\documentclass[twoside]{article}

\renewcommand{\vec}[1]{\mathbf{#1}}

\usepackage{subfigure}
\usepackage{algpseudocode}
\usepackage{amsmath, amsthm, amssymb, amstext, comment, graphicx, fullpage}
\usepackage{fullpage}
\usepackage{mathtools,extarrows}
\usepackage{hyperref}
\usepackage{bbm}
\usepackage{array}
\usepackage{xr}

\usepackage[ruled,linesnumbered]{algorithm2e} % For algorithms

\usepackage{xcolor}

\renewcommand{\vec}[1]{\mathbf{#1}}

\newcommand{\norm}[1]{\left\| #1 \right\|}
\newcommand{\abss}[1]{\left\lvert {#1} \right\rvert}

\newtheorem{theorem}{Theorem}[section]

\newtheorem{corollary}{Corollary}[section]

\newtheorem{lemma}{Lemma}[section]

\newtheorem{assumption}{Assumption}[section]

\newtheorem{remark}{Remark}[section]

\allowdisplaybreaks

\title{Logistic regression with peer-group effects\\via inference in higher-order Ising models}
\author{
	Constantinos Daskalakis\\
	EECS \& CSAIL, MIT\\
	\tt{costis@csail.mit.edu}
	\and
	 Nishanth Dikkala\\
	EECS \& CSAIL, MIT\\
	\tt{nishanthd@csail.mit.edu}
	\\\and
	Ioannis Panageas\\
	ISTD \& SUTD\\
	\tt{ioannis@sutd.edu.sg}
}

\begin{document}
% note that the abstract must come before \maketitle
%\begin{abstract}
%\end{abstract}
\date{}
	\maketitle
\begin{abstract}
Spin glass models, such as the Sherrington-Kirkpatrick, Hopfield and Ising models, are all well-studied members of the exponential family of discrete distributions, and have been influential in a number of application domains where they are used to model correlation phenomena on networks. Conventionally these models have quadratic sufficient statistics and consequently capture correlations arising from pairwise interactions. In this work we study extensions of these to models with higher-order sufficient statistics, modeling behavior on a social network with peer-group effects. In particular, we model binary outcomes on a network as a higher-order spin glass, where the behavior of an individual depends on a linear function of their own vector of covariates and some polynomial function of the behavior of others, capturing peer-group effects. Using a {\em single}, high-dimensional sample from such model our goal is to recover the coefficients of the linear function as well as the strength of the peer-group effects. The heart of our result is a novel approach for showing strong concavity of the log pseudo-likelihood of the model, implying statistical error rate of $\sqrt{d/n}$ for the Maximum Pseudo-Likelihood Estimator (MPLE), where $d$ is the dimensionality of the covariate vectors and $n$ is the size of the network (number of nodes). Our model generalizes vanilla logistic regression  as well as the models studied in recent works of ~\cite{chatterjee2007estimation,ghosal2018joint,DDP19}, and our results extend these results to accommodate higher-order interactions.
\end{abstract}

%OLD abstract: Spin glass models such as the Sherrington-Kirkpatrick model, Hopfield model and the Ising model are all well-studied popular members of the exponential family of distributions which have been applied to a number of applications where we wish to model correlations using an underlying graph.
%
%All of these models typically are studied with a quadratic sufficient statistic (i.e. pairwise correlation structure), however in many cases higher-order correlations do exist.
%In this work, we study the problem of parameter inference on Ising models with a degree $m$ sufficient statistic from a single sample from the model. We show consistency for a pseudo-likelihood based estimator extending the work of Chatterjee, Ghosal Mukherjee, Daskalakis et al.
%We do this by a novel approach for showing strong concavity of the log pseudo-likelihood which allows us to argue about such properties for models with higher-order correlations and potentially bears application to other settings as well.
%

\section{Introduction}
\label{sec:intro}

Did you choose \textcolor{red}{red} rather than \textcolor{blue}{blue} because some inherent attributes of yours biased you towards \textcolor{red}{red}, or because your social environment biased you towards that color? Of course, the answer is typically ``\textcolor{red}{bo}\textcolor{blue}{th}.'' Indeed, a long literature in econometrics and the social sciences has substantiated the importance of peer effects in network behavior in topics as diverse as criminal activity (see e.g.~\cite{glaeser1996crime}), welfare participation (see e.g.~\cite{bertrand2000network}), school achievement (see e.g.~\cite{sacerdote2001peer}), participation in retirement plans (see e.g.~\cite{duflo2003role}), and obesity (see e.g.~\cite{trogdon2008peer,christakis2013social}).
%A prominent dataset where peer effects have been studied are data collected by the National Longitudinal Study of Adolescent Health, a.k.a.~AddHealth study~\cite{harris2009waves}. This was a major national study of students in grades 7-12, who were asked to name their friends---up to 10, so that friendship networks can be constructed, and answer hundreds of questions about their personal and school life, and it also recorded information such as the age, gender, race, socio-economic background, and health of the students.
On the other hand, estimating the mechanisms through which peer and individual effects  drive behavior in such settings 
has been quite challenging; see~e.g.~\cite{manski1993identification,bramoulle2009identification}.

From a modeling perspective, a class of probabilistic models that are commonly used to model binary behavior in social networks are spin glass models, such as the well-studied Sherrington-Kirkpatrick, Hopfield and Ising models. In these models, a vector of binary behaviors $\vec{y} \in \{-1,1\}^V$ across all nodes of some network $G=(V,E)$ is sampled jointly according to the Gibbs distribution, $p(\vec{y})={1 \over Z} \exp(-{\rm En}(\vec{y}))$, defined by some energy function ${\rm En}(\vec{y})$ of the aggregate behavior, where the functional form of ${\rm En}(\cdot)$ typically depends on characteristics of the nodes as well as the structure of their social network. Such models studied originally in Statistical Physics, have found myriad applications in diverse fields, including Probability Theory, Markov Chain Monte Carlo, Computer Vision, Computational Biology,  Game Theory, and, related to our focus, Economics and the Social Sciences~\cite{LevinPW09,Chatterjee05,Felsenstein04,DaskalakisMR11,GemanG86,Ellison93,MontanariS10}.

Closely related to our work, a series of recent works have studied estimation of spin glass models incorporating both peer and individual effects as drivers of behavior~\cite{Chatterjee07,ghosal2018joint,DDP19}. Generalizing the classical logistic regression model, these works consider models of binary behavior on a  network, conforming to the following general class of models. Suppose that the nodes of a social network $G=(V,E)$ have individual characteristics $\vec{x}_i \in \mathbb{R}^d$, $i \in V$, and sample binary behaviors $\vec{y} \in \{\pm 1\}^V$ according to some measure that combines individual and peer effects, taking the following form:
\begin{align}\Pr[\vec{y}] = {1 \over Z_{\theta,\beta}} \exp\left(\sum_{i \in V}(\theta^{\top} \vec{x}_i) y_i + \beta\cdot f(\vec{y})\right), \label{eq:the model}
\end{align}
where a linear function $\theta^{\top} \vec{x}_i$ of node $i$'s individual characteristics determines the ``external field'' on that node, i.e.~the direction and strength of the ``local push'' of that node towards $-1$ or $+1$, and some function $f(\vec{y})$ of the nodes' joint behavior expresses what configurations in $\{\pm 1 \}^V$ are encouraged by peer-group effects. In particular, setting $\beta=0$ recovers the standard logistic regression model, where nodes choose their behaviors independently, but setting $\beta>0$ incorporates peer-group effects, as expressed by $f$. Without loss of generality, $f$ is a multi-linear function, and we can take $E$ to contain a hyperedge for each monomial in $f$, i.e.~take $f(\vec{y})=\sum_{\vec{e} \in E} w_{\vec{e}} \vec{y}_{\vec{e}}$ where $\vec{y}_{\vec{e}} = \prod_{i \in \vec{e}} y_i$.

Given a collection $\vec{x}_1,\ldots,\vec{x}_n \in \mathbb{R}^d$ of covariates, some function $f: \{\pm 1\}^V \rightarrow \mathbb{R}$, and a {\em single} sample $\vec{y}$  drawn from a model conforming to~\eqref{eq:the model}, the afore-cited works of Chatterjee~\cite{Chatterjee07}, Ghosal and Mukerjee~\cite{ghosal2018joint} and Daskalakis et al.~\cite{DDP19} provide computationally and statistically efficient algorithms for estimating $\theta$ and $\beta$. Specifically, these works study the restriction of model~\eqref{eq:the model} to the case where $f$ contains only pair-wise effects, i.e.~where function $f$ is a multilinear function of degree $2$. In particular, Chatterjee~\cite{Chatterjee07} studies the case where $\theta=0$ and $f$ is bilinear, Ghosal and Mukerjee~\cite{ghosal2018joint} the case where $d=1$, all $x_i$'s equal $1$, and $f$ is bilinear, while Daskalakis et al.~\cite{DDP19}  the general bilinear case. Extending these works, the goal of our work is to provide computationally and statistically efficient estimation methods for models where $f$ has peer effects of higher-order. As such, our new methods can accommodate richer models, capturing a much broader range of social interactions, e.g.~settings where nodes belong in various groups, and dislike fragile majorities in the groups they belong to. Our main result is the following.
%
 %drawn according to measure~\eqref{eq:the model}, our goal is to output an estimate $(\hat{\theta},\hat{\beta})$ of the true parameters $(\theta,\beta)$.

\begin{theorem}[Informal] \label{thm:informalone}
Let $G=(V,E,w:E \rightarrow \mathbb{R})$ be a weighted hypergraph with edges of cardinality at least two and at most some constant $m$, and let $f(\vec{y})=\sum_{\vec{e} \in E} w_{\vec{e}} \vec{y}_{\vec{e}}$. Assume that each vertex has bounded degree (Assumption \ref{eq:bounded}) and the hypergraph is dense enough (Assumption \ref{eq:lowebounded}). Moreover, assume that the true parameters $\theta_0, \beta_0$ and the feature vectors have bounded $\ell_2$ norm, and the empirical covariance matrix of the feature vectors has singular values upper and lower bounded by constants (Assumption \ref{eq:assumption3}). Then, there exists a polynomial-time algorithm, which, given a single sample from model~\eqref{eq:the model}, outputs an estimate $(\tilde{\theta}, \tilde{\beta})$ such that $\norm{(\tilde{\theta},\tilde{\beta})-(\theta_0,\beta_0)}_2$ is  $O\left(\sqrt{\frac{d}{n}}\right)$, with probability at least $99\%$, where $n=|V|$.
\end{theorem}

\paragraph{Discussion of Main Result.} First, let us discuss the assumptions made in our statement. Note that the assumptions about $\theta$ and the $x_i$'s are standard, and are commonly made even for vanilla logistic regression without peer effects ($\beta=0$). The assumption about the boundedness of $\beta$ and the degree of the hypergraph is needed so that the peer-group effects do not overwhelm the individual effects, making $\theta$ non-identifiable. Finally, the assumption on the density of the hypergraph is needed so that the individual effects do not overwhelm the peer-group effects, making $\beta$ non-identifiable. Our assumptions about $\beta$ and the hypergraph are generalizations of corresponding assumptions made in prior work. As such, our main result is a direct generalization of prior work to accommodate higher-order peer effects.

We should also discuss the importance, in both our work and the  work we build upon~\cite{Chatterjee07,ghosal2018joint,DDP19}, of estimating the parameters of our model using a {\em single} sample, which stands in contrast to other recent work studying estimation of Ising models and more general Markov Random Fields from  multiple samples; see e.g.~\cite{BreslerGS14,Bresler15,BreslerK16,vuffray2016interaction,klivans2017learning,BKM19}. The importance of estimating from a single sample arises from the applications motivating our work, where it is more common than not that we really only have a single sample of node behavior across the whole network, and  cannot  obtain a fresh independent sample of behavior tomorrow or within a reasonable time-frame.

\paragraph{Techniques.} Towards obtaining Theorem~\ref{thm:informalone}, we encounter several technical challenges. A natural approach is to use our single sample to perform Maximum Likelihood Estimation. However, this approach faces two important challenges. First, it has been shown that the single sample Maximum Likelihood Estimator is not necessarily consistent~\cite{Chatterjee07}. Second, the likelihood function involves the partition function $Z_{\theta,\beta}$, which is generally computationally intractable to compute. In view of these issues, we follow instead the approach followed in prior work. Rather than maximizing the likelihood of the sample, we maximize its {\em pseudolikelihood}, defined as $\prod_i \Pr[y_i~|~\vec{y}_{-i}]$. This concave function of our parameters $\theta$ and $\beta$ is computationally easy to optimize, however we need to show that its maximum is consistent. To argue this we establish two main properties of the log-pseudolikelihood: (i) the log-pseudolikelihood is strongly concave in the neighborhood of its maximum; and (ii)  its gradient at the true model parameters is bounded. As both the Hessian and the gradient of log-pseudolikelihood are functions of the vector of variables $\vec{y}$, which are jointly sampled, to argue (i) and (ii)  we need to control functions of dependent random variables. To do this we use exchangeable pairs,  adapting the technique of ~\cite{chatterjee2016nonlinear}, combined with a parity argument on $G$ and $f$'s partial derivatives. In turn, (i) and (ii) suffice to establish the consistency of the Maximum Pseudolikelihood Estimator (MPLE).

\subsection{More Related Work}
\label{sec:related}
Learning and testing questions on Ising models have been widely studied in diverse contexts. A popular instantiation of the learning problem is structure learning, where given access to multiple i.i.d. samples from the model we wish to infer the underlying graph's structure. This was first studied for tree graphical models by \cite{chow1968approximating} and has since then seen a lot of work both in terms of upper bounds and lower bounds side~\cite{santhanam2012information}. More recently, \cite{bresler2015efficiently} gave a striking algorithm for structure learning in bounded degree graphs which required samples only logarithmic in the number of nodes of the graph. The running time and sample complexity of this approach was improved in later works of \cite{vuffray2016interaction,klivans2017learning,hamilton2017information}. The works of \cite{klivans2017learning,hamilton2017information} provide learning results for MRFs with higher-order interactions on alphabet of sizes larger than 2. Property testing questions on Ising models have also been studied by \cite{DaskalakisDK18}. All of the above works, however, make use of access to many independent samples from a Ising model. Closer to the model we consider in this paper is the line of work initiated by \cite{chatterjee2007estimation} and extensions in the works of \cite{bhattacharya2018inference,ghosal2018joint,DDP19} wherein we try to infer an Ising model described by a few parameters using a single sample from the model. \cite{bresler2018optimal,mukherjee2018global} study hypothesis testing questions on the Ising model from a single sample. 
%study testing for the existence of an external field in the model given access to a single sample.

%\nishanth{Should we re-iterate the econometrics references here?}

%Pointers:
%\begin{enumerate}
%	\item Inference in Ising Models from single sample related work
%	\item Ising model learnin/testing from many samples related work.
%	\item Econometrics peer effects work
%	\item
%\end{enumerate} 

\section{Preliminaries}
\label{sec:prelim}
We use bold letters such as $\vec{x}, \vec{y}$ to denote vectors and capital letters $A,W$ to denote matrices. All vectors are assumed to be column vectors, i.e. $\text{dim} \times 1$ (except when we refer to the parameters as $(\theta,\beta)$ instead of $(\theta^{\top},\beta)$). We will refer to $W_{ij}$ as the $(i,j)^{th}$ entry of matrix $W$.
We will use the following matrix norms. For a $n \times n$ matrix $W$,
\begin{equation}
\norm{W}_2 = \max_{\norm{x}_2 = 1} \norm{Wx}_2,  \: \:  \norm{W}_{\infty} = \max_{j \in [n]} \sum_{i=1}^n \abss{W_{ij}} ,
\; \; \norm{W}_F = \sqrt{\sum_{i=1}^n \sum_{j=1}^n W_{ij}^2}.
\end{equation}
When $W$ is a symmetric matrix we have that $\norm{W}_2 \le \norm{W}_{\infty} \le \norm{W}_F \le \sqrt{n}\norm{W}_2 \le \sqrt{n}\norm{W}_\infty$ and in general we have $\norm{W}_2^2 \leq \norm{W}_{\infty} \norm{W}_1$.

We use $\lambda$ to denote eigenvalues of a matrix and $\sigma$ to denote its singular values. $\lambda_{\min}$ refers to the smallest eigenvalue and $\lambda_{\max}$ to the largest, and similar notation is used for the singular values. We use $\vec{e}$ or a collection $\{z_1,...,z_m\}$ to denote a hyperdge and moreover its weight is denoted by $w_{\vec{e}}$ or $w_{(z_1,...,z_m)}$.

We will say an estimator $\hat{\theta}_n$ is consistent with a rate $r(n)$ (or equivalently $r(n)$-consistent) with respect to the true parameter $\theta_0$ if there exists an integer $n_0$ and a constant $C > 0$ such that for every $n > n_0$, with probability at least $99\%$,
\begin{align*}
\norm{\hat{\theta}_n - \theta_0 }_2 \le  \frac{C}{r(n)}.
\end{align*}

\subsection{Ising Model and Inference}
The Ising model is a well-studied binary graphical model. We provide the description of the model here.

\begin{enumerate}
	\item \textbf{Ising Model (simple):} Given a weighted undirected graph $G(V,E)$ with $|V|=n$ and a $n \times n$ weight matrix $W$ and assignment $\vec{\sigma} : V \to \{-1,+1\}$, an Ising model is the following probability distribution on the $2^n$  configurations of $\vec{\sigma}$:
\begin{equation}
\Pr\{\vec{y} = \sigma\} = \frac{\exp\left(\sum_{v \in V}h_v \sigma_v + \beta\vec{\sigma}^{\top}W \vec{\sigma}\right)}{Z_G} \label{eq:ising}
\end{equation}
where $$Z_G = \sum_{\vec{\tilde{\sigma}}} \exp\left(\sum_{v \in V}h_v \tilde{\sigma}_v + \beta\vec{\tilde{\sigma}}^{\top}W \vec{\tilde{\sigma}}\right)$$ is the partition function of the system (or renormalization factor). Moreover the term $\sum_v h_v \sigma_v$ is called the external field and $\beta$ is called the inverse temperature. It can be observed that, without loss of generality, we can restrict the matrix $W$ to have zeros on its diagonal.
	\item \textbf{Ising Model (Hypergraph):} Given a hypergraph graph $G(V,E)$ (each edge $\vec{e}$ has at most $m$ incident vertices and at least two), weights $w_{\vec{e}}$ and assignment $\vec{\sigma} : V \to \{-1,+1\}$, an Ising model is the following probability distribution on the $2^n$  configurations of $\vec{\sigma}$:
\begin{equation}
\Pr\{\vec{y} = \sigma\} = \frac{\exp\left(\sum_{v \in V}h_v \sigma_v + \beta f(\vec{\sigma})\right)}{Z_G} \label{eq:ising},
\end{equation}
where $f(\vec{\sigma}) = \sum_{\vec{e} \in E(G)}w_{\vec{e}}\vec{\sigma}_{\vec{e}}$ and $\vec{\sigma}_{\vec{e}} = \prod_{v \in \vec{e}} \sigma_v$. Observe that $f(\vec{\sigma})$ is a multilinear polynomial of degree $m$ (since $y^2_v=1$ for all $v$ and every realization, weighted hypergraphs capture all distributions with $f$ a polynomial function).
\end{enumerate}

\paragraph{Inference of Ising models with Hypergraphs:}
In this paper we focus on the following modification of the Ising model for hypergraphs. It is assumed that we are given \textbf{one sample} from the following distribution:
\[
\Pr[\vec{y} = \vec{\sigma}] = \frac{\exp(\beta f(\vec{\sigma})+ \sum_v (\vec{x}_v^{\top} \theta) \sigma_v)}{Z_G(\beta,\theta)},
\]
where $\beta, \vec{\theta}$ are unknown parameters, $f: \{-1,+1\}^n \to \mathbb{R}$ is a polynomial (multilinear) function and each summand is of degree at most $m$ and at least two ($Z_G(\beta,\theta)$ is the renormalization factor again). The goal is to \textit{estimate} the parameters $\beta$ and $\theta$. This problem is a generalization of the logistic regression model with dependent observations problem as appeared in \cite{DDP19} (for $m=2$), applied to hypergraphs.
\begin{itemize}
\item Observe that for each index $v$  we can write $f(\vec{y}) = y_v f_v(\vec{y}_{-v}) + f_{-v}(\vec{y}_{-v})$ (both $f_{-v}, f_v$ are multilinear functions that do not depend on $y_v$). It is easy to see that $f_v (\vec{y}_{-v}) = \frac{\partial f}{\partial y_v}$. Each hyperedge $\vec{e}$ is a collection of at most $m$ vertices $v \in V$. One may write $\vec{y}_{\vec{e}} = \prod_{v \in \vec{e}} y_v$ and moreover $f(\vec{y}) = \sum_{\vec{e} \in E} w_{\vec{e}} \vec{y}_{\vec{e}}$ and $y_v f_v(\vec{y}_{-v}) = \sum_{\vec{e} \in E, v \in \vec{e}} w_{\vec{e}} y_{\vec{e}}$.
\item For all vertices $v$ and $\sigma_v \in \{\pm 1\}$, conditioning on a realization of the response variables $\vec{y}_{-v}$:
\begin{align}
\Pr[y_v = \sigma_v] &= {1 \over 1 + \exp\left( -2\left(\theta^{\top} \vec{x}_v + \beta f_v(\vec{y}_{-v})\right)\sigma_v\right)}. \label{eq:costas logistic}
\end{align}
\item Interpretation: The probability that the conditional distribution of $y_v$ assigns to $+1$ is determined by the logistic function applied to $2\left(\theta^{\top} \vec{x}_v  + \beta f_{v}(\vec{y}_{-v})\right)$ instead of $2\theta^{\top} \vec{x}_v$.
\end{itemize}
% note: this command has been disabled to remove the ACM copyright block. Sorry...
\subsection{Assumptions}
Our Assumptions can be listed below:
\begin{assumption}[Bounded degree]\label{eq:bounded}
\begin{equation}
\sum_{\vec{e}: i \in \vec{e}} |w_{\vec{e}}| \leq 1,
\end{equation}
for all vertices $i$, where $\vec{e}$ captures the hyperedges. The number one on the R.H.S can be replaced with any constant. This assumption is mainly used in our concentration bounds.
\end{assumption}
\begin{assumption}[Enough weight at the hyperedges]\label{eq:lowebounded}
\begin{equation}
\sum_{\vec{e}\in E, \atop |e| = m} w_e^2 \textrm{ is }\Omega(n),
\end{equation}
This assumption is mainly used to prove strong concavity of the pseudolikelihood for the estimation of $\beta$.
%In particular, if $\sum_{\vec{e}} w_{\vec{e}}^2$ is very small then $\beta_0$ a nonzero constant is indistinguishable from $\beta_0 = 0$.
\end{assumption}
\begin{assumption}[Parameters and features]\label{eq:assumption3} The true parameter $\beta_0$ belongs in some interval $(-B,B)$ and $\norm{\theta_0}_2 < \Theta$ for some known constants $B, \Theta$ that are independent of $n,d$. We denote by $\mathbb{B} \subseteq \mathbb{R}^{d+1}$, $\mathbb{B} = \{(\theta,\beta) \in \mathbb{R}^{d+1}, |\beta| \leq B, \norm{\theta}_2 \leq \Theta\}$ (i.e., the closure of the set that the parameters may belong to).

Moreover for every feature vector $\vec{x}_v$ we have $\norm{\vec{x}_v}_2 \leq M$ (for some known constant $M$ independent of $n,d$). Finally, the covariance matrix (of size $d\times d$) of the feature vectors, i.e., $\frac{1}{n} X^{\top}X$ where $X^{\top} = \left(\vec{x}_1 \; \vec{x}_2 \ldots \vec{x}_n\right)$ has minimum and maximum eigenvalues bounded by constants (independent of $n,d$) and the projection matrix $F = I - X(X^{\top}X)^{-1}X^{\top}$ satisfies $\norm{F}_{\infty}$ is bounded by a constant (one without loss of generality).
\end{assumption}
\subsection{Pseudo-Likelihood - Gradient and Hessian}\label{sec:pseudolike}
The pseudolikelihood as defined by Chatterjee in \cite{chatterjee2007estimation} for a simpler model and instantiated in our model is given by the following expression:
\begin{equation}\label{eq:pseudolikelihood}
PL(\theta,\beta) := \left(\prod_{i=1}^n \Pr[y_{i}\big| \vec{y}_{-i}]\right)^{1/n} = \left(\prod_{i=1}^n\frac{\exp((\theta^{\top} \vec{x}_i + \beta f_{i}(\vec{y}_{-i}))y_i)}{\exp(\theta^{\top} \vec{x}_i + \beta f_{i}(\vec{y}_{-i}))+\exp(-\theta^{\top} \vec{x}_i - \beta f_{i}(\vec{y}_{-i}))}\right)^{1/n}
\end{equation}
Taking the log, the log pseudolikelihood for a specific sample $\vec{y}$ is given by:
\begin{equation}
LPL(\theta,\beta) := \frac{1}{n}\sum_{i=1}^n \left[ y_i \beta f_{i}(\vec{y}_{-i})+y_i(\theta^{\top}\vec{x}_i) - \ln \cosh (\beta f_i(\vec{y}_{-i})+\theta^{\top}\vec{x}_i)\right] - \ln 2,
\end{equation}

The first order conditions give:
\begin{equation}\label{eq:foc}
\begin{array}{ll}
\frac{\partial LPL(\theta,\beta)}{\partial \beta} = \frac{1}{n}\sum_{i=1}^n \left[ y_i f_{i}(\vec{y}_{-i}) -  f_i(\vec{y}_{-i})\tanh (\beta f_i(\vec{y}_{-i})+\theta^{\top}\vec{x}_i)\right]=0 ,\\
\frac{\partial LPL(\theta,\beta)}{\partial \theta_k} = \frac{1}{n}\sum_{i=1}^n \left[ y_ix_{i,k} - x_{i,k}\tanh (\beta f_i(\vec{y}_{-i})+\theta^{\top}\vec{x}_i)\right]=0.
\end{array}
\end{equation}
The solution to equation~(\ref{eq:foc}) is called Maximum Pseudolikelihood Estimator (Hessian is negative semidefinite, see below) and is denoted by $(\hat{\theta},\hat{\beta})$ or $(\hat{\theta}_{MPL},\hat{\beta}_{MPL})$.

The Hessian $H_{(\theta, \beta)}$ of the log-pseudolikelihood is given by:
\begin{equation}
\begin{array}{ll}
\frac{\partial^2 LPL(\theta,\beta)}{\partial \beta^2} = -\frac{1}{n}\sum_{i=1}^n \frac{f_i^2(\vec{y}_{-i})}{\cosh^2 (\beta f_i(\vec{y}_{-i})+\theta^{\top}\vec{x}_i)},\\
\frac{\partial^2 LPL(\theta,\beta)}{\partial \beta \partial \theta_k} = -\frac{1}{n}\sum_{i=1}^n   \frac{x_{i,k}f_i(\vec{y}_{-i})}{\cosh^2 (\beta f_i(\vec{y}_{-i})+\theta^{\top}\vec{x}_i)},\\
\frac{\partial^2 LPL(\theta,\beta)}{\partial \theta_l \partial \theta_k} = -\frac{1}{n}\sum_{i=1}^n   \frac{x_{i,l}x_{i,k}}{\cosh^2 (\beta f_i(\vec{y}_{-i})+\theta^{\top}\vec{x}_i)}.
\end{array}
\end{equation}
Writing the Hessian differently we get  $$H_{(\theta,\beta)} = - \frac{1}{n}\sum_{i=1}^n \frac{1}{\cosh^2 (\beta f_i(\vec{y}_{-i})+\theta^{\top}\vec{x}_i)}X_iX_i^{\top}$$ where $X_i = (\vec{x}_i^{\top}, f_i(\vec{y}_{-i}))^{\top}$. Thus $-H$ is a positive semidefinite matrix and $LPL$ is concave. Moreover if $(\theta, \beta)$ satisfies Assumptions \ref{eq:bounded} and \ref{eq:assumption3} it follows that
\begin{equation}\label{eq:inequalityleft}
\begin{array}{ll}
\frac{1}{\cosh^2 (B+M\cdot \Theta)} \cdot \left( \frac{1}{n}\sum_{i=1}^n X_iX_i^{\top} \right)\preceq -H_{(\theta, \beta)} \preceq \left( \frac{1}{n}\sum_{i=1}^n X_iX_i^{\top} \right).
\end{array}
\end{equation}
\begin{remark}[LPL is smooth]\label{rem:smooth}
Since $\norm{X_i}^2_2 = \norm{\vec{x}_i}_2^2 + f_i^2(\vec{y}_{-i}) \leq \Theta^2 +1$ (assuming Assumption \ref{eq:bounded} trivially holds $|f_i(\vec{y}_{-i})| \leq 1$) it holds that $\lambda_{\max}(-H_{(\theta, \beta)}) \leq \Theta^2 +1$ for all $(\theta,\beta) \in \mathbb{R}^{d+1}$ which satisfy Assumption \ref{eq:assumption3}, hence $-LPL$ is a $\Theta^2+1$-smooth function, i.e. $-\nabla LPL$ is $\Theta^2 +1$-Lipschitz.
\end{remark}

We conclude this session with an important lemma that explains the reason we need the technical lemmas in Section 3 and involves that gradient and the Hessian of the log-psudolikelihood (appeared in \cite{DDP19}).
\begin{lemma}[Consistency of the MPLE \cite{DDP19}]
	\label{lem:logistic-consistency}
Let $(\theta_0,\beta_0)$ be the true parameter. We define $(\theta_t,\beta_t) = (1-t)(\theta_0,\beta_0)+ t(\hat{\theta}_{MPL},\hat{\beta}_{MPL})$ and let $\mathcal{D} \in [0,1]$ be the largest value such that $(\theta_{\mathcal{D}}, \beta_{\mathcal{D}}) \in \mathbb{B}$ (if it does not intersect the boundary of $\mathbb{B}$, then $\mathcal{D}=1$), where $\mathbb{B}$ is defined in Assumption \ref{eq:assumption3}.
Then,
\begin{align*}
  &\norm{\nabla LPL(\theta_0,\beta_0)}_2 \ge \mathcal{D} \min_{(\theta,\beta) \in \mathbb{B}}\lambda_{\min}\left(-H_{(\theta, \beta)}\right) \norm{(\theta_0 - \hat{\theta}_{MPL},\beta_0 - \hat{\beta}_{MPL})}_2 \\
  &= \min_{(\theta,\beta) \in \mathbb{B}}\lambda_{\min}\left(-H_{(\theta, \beta)}\right) \norm{(\theta_0 - \theta_{\mathcal{D}},\beta_0 - \beta_{\mathcal{D}})}_2.
\end{align*}	
\end{lemma}

To prove the main result, we apply Lemma \ref{lem:logistic-consistency} by showing: (in the rest of the paper)
\begin{enumerate}
	\item \textit{A concentration result} for $\norm{\nabla LPL(\theta_0,\beta_0)}^2_2$ around $d/n$ (Section \ref{sec:concentration}) which in words gives that the gradient of the log-pseudolikelihood at the true parameter is small (note that it is zero at the MPLE) (I).
	\item \textit{A lower bound (positive constant that depends on the degree of polynomial $f$)} for $\min_{(\theta,\beta) \in \mathbb{B}}\lambda_{\min}\left(-H_{(\theta, \beta)}\right)$ (Section \ref{sec:convexity}) with high probability (II).
\end{enumerate}
We combine the above with the observation that $\mathcal{D} = 1$ for $n$ sufficiently large. This is true because $\norm{(\theta_{\mathcal{D}} - \theta_0,\beta_{\mathcal{D}}-\beta_0)}_2 \to 0$ as $n \to \infty$ (is of order $\frac{1}{\sqrt{n}}$ and that any point on the boundary of $\mathbb{B}$ has a fixed (independent of $n$) positive distance to $(\theta_0,\beta_0)$ since $(\theta_0,\beta_0)$ lies in the interior of $\mathbb{B}$.

This gives the desired rate of consistency which we show in Section~\ref{sec:alltogether}. 
\section{Maximum Pseudo-Likelihood (MPLE): Concentration and Strong Concavity}
In this section, we prove Theorem~\ref{thm:informalone}. In words, we show consistency of the MPLE which we prove via bullets (I), (II) and then applying Lemma~\ref{lem:logistic-consistency} as stated in the previous section. Our main result is formally given below:
\begin{theorem}[Main (Formal)]
	\label{thm:logistic}
	Consider the model of (\ref{eq:the model}) with Assumptions \ref{eq:bounded}, \ref{eq:lowebounded}, \ref{eq:assumption3} and denote Maximum Pseudo-Likelihood Estimate (MPLE) with $(\hat{\theta}_{MPL},\hat{\beta}_{MPL})$.  With probability $99.9\%$ it holds that \[\norm{(\hat{\theta}_{MPL},\hat{\beta}_{MPL}) - (\theta_0,\beta_0)}_2 \le  O\left({\sqrt{d \over n}}\right) 2^{O(m)}\] and we can compute an estimate with the same order of consistency in $O(\ln n)$ iterations of projected gradient descent (Algorithm in Section \ref{sec:gd-analysis}) where each iteration takes polynomial (in $n$) time.
\end{theorem}

\subsection{Concentration Results for Gradient (I)}\label{sec:concentration}
The first main technical Lemma is to show that the norm of the gradient of the log-pseudolikelihood is small enough at the true parameters (Corollary \ref{cor:smallgradient}). This is necessary because we are working with the finite sample pseudolikelihood (empirical). In what follows we show that the difference between sum of $y_i f(y_{-i})$ (or $y_i \vec{x}_i$) and the sum of their conditional expectations is small.
\begin{lemma}[Variance Bound 1]
	\label{lem:conc1}
It holds that
\begin{align*}
&\mathbb{E}_{\theta_0,\beta_0} \left[ \left(\sum_{i=1}^n y_i f_i(\vec{y}_{-i})- f_i(\vec{y}_{-i})\tanh (\beta_0 f_i(\vec{y}_{-i})+\theta_0^{\top}\vec{x}_i)\right)^2\right] \leq (12+4B)(m-1)n.
\end{align*}
\end{lemma}

\begin{lemma}[Variance Bound 2]
	\label{lem:conc2}
It holds that
\begin{align*}
&\mathbb{E}_{\theta_0,\beta_0} \left[\sum_{k=1}^d \left (\sum_{i=1}^n x_{i,k}y_i - x_{i,k}\tanh (\beta_0 f_i(\vec{y}_{-i})+\theta_0^{\top}\vec{x}_i)\right)^2 \right] \leq (1+B)4M^2 \cdot (m-1)d n.
\end{align*}
\end{lemma}
We are now ready to prove bullet (I).
\begin{corollary}\label{cor:smallgradient}
For each $\delta>0$ and $n$ sufficiently large, with probability $1-\delta$ it holds that
\[\Pr_{\theta_0,\beta_0} \left[\norm{\nabla LPL(\theta_0,\beta_0)}_2 \leq C\sqrt{\frac{1}{\delta}}\sqrt{\frac{d}{n}}\right],
\]
for some global constant $C$.
\end{corollary}
\begin{proof}
Observe that (see Equations of the gradient, left-hand side in (\ref{eq:foc}))
\begin{equation}
\begin{array}{cc}
\norm{\nabla LPL(\theta_0,\beta_0)}_2^2 &= \frac{1}{n^2}\sum_{k=1}^d \left (\sum_{i=1}^n x_{i,k}y_i - x_{i,k}\tanh (\beta_0 f_i(\vec{y}_{-i})+\theta_0^{\top}\vec{x}_i)\right)^2\\&+ \frac{1}{n^2}\left(\sum_{i=1}^n y_i f_i(\vec{y}_{-i})- f_i(\vec{y}_{-i})\tanh (\beta_0 f_i(\vec{y}_{-i})+\theta_0^{\top}\vec{x}_i)\right)^2.
\end{array}
\end{equation}
The claim is an application of Lemmas \ref{lem:conc1}, \ref{lem:conc2} and Markov's inequality.
\end{proof}
\subsection{Strong Concavity of log-Pseudolikelihood (II)}\label{sec:convexity}
\paragraph{Schur's complement.}
Let \[X^{\top} = \left(\vec{x}_1 \; \vec{x}_2 \ldots \vec{x}_n\right),\] which is the matrix of the covariates (of size $d \times n$).
Using Equation (\ref{eq:inequalityleft}) (the negative Hessian of log-Pseudolikelihood dominates the matrix below) we get
	\[
	-H \succeq \frac{1}{\cosh^2 (B+M\cdot \Theta)} G \textrm{ where } G:=
	 \left(
	\begin{array}{cc}
	\frac{1}{n}X^{\top}X & \frac{1}{n}X^{\top}\vec{f}\\
	\frac{1}{n}\vec{f}^{\top}X &\frac{1}{n}\norm{\vec{f}}_2^2
	\end{array}
	\right),\] and $\vec{f} := (f_1 (\vec{y}_{-1}),...,f_n (\vec{y}_{-n}))$.
	
    We set $Q = \frac{1}{n}X^{\top} X$ and use the properties of Schur complement on the matrix
    \[G - \lambda I = \left(
    \begin{array}{cc}
	Q - \lambda I & \frac{1}{n}X^{\top}\vec{f}\\
	\frac{1}{n}\vec{f}^{\top}X &\frac{1}{n}\norm{\vec{f}}_2^2 - \lambda
	\end{array}
	\right)\]
to get that
	\begin{equation}
	\det\left(G-\lambda I\right) = \det \left(Q - \lambda I\right) \det \left(\frac{1}{n}\vec{f}^{\top}\left(I - \frac{1}{n}X\left(Q - \lambda I \right)^{-1}X^{\top}\right)\vec{f}-\lambda\right).
	\end{equation}
	Therefore the minimum eigenvalue of $G$ is at least a positive constant as long as the minimum eigenvalues of \[Q \textrm{ and }\frac{1}{n}\vec{f}^{\top}\left(I - \frac{1}{n}XQ^{-1}X^{\top}\right)\vec{f}\] are at least positive constants independent of $n,d$.
	Recall from our assumptions (Assumption \ref{eq:assumption3}) we have that $\lambda_{\min}(Q) \ge c_1$ always where $c_1$ is a positive constant independent of $n,d$. Hence, it remains to show that \[\lambda_{\min}\left(\frac{1}{n}\vec{f}^{\top}\left(I - \frac{1}{n}XQ^{-1}X^{\top}\right)\vec{f}\right) \ge c_2\] for a positive constant $c_2$ with high probability (with respect to the randomness in drawing $\vec{y}$).

Denoting $F =  I - X(X^{\top}X)^{-1}X^{\top} = I - \frac{1}{n} XQ^{-1}X^{\top}$, observe that $F$ has the property $F^2 = F$ (i.e. is idempotent) and hence all the eigenvalues of $F$ are $0,1$ (since is of rank $n-d$, it has $d$ eigenvalues zero and $n-d$ eigenvalues one).
Our goal is to show that
\begin{lemma}\label{lem:main2}
\begin{equation}\label{eq:energy}
 \vec{f}^{\top}F\vec{f}=\norm{F\vec{f}}_2^2 \ge c_2 n \; \;\text{ with probability $1-o(1)$},
 \end{equation}
 where the probability is with respect to the randomness in drawing $\vec{y}$.
\end{lemma}
%Moreover, from the sub-multiplicativity of the spectral norm, it holds that
%	\begin{align}
%	\norm{FA}_2 \leq \norm{F}_2 \times \norm{A}_2 \leq 1. \label{eq:fa2}
%	\end{align}
%	We also have that $\norm{FA}_{F}^2$ is $\Omega(n)$. This is because $\sigma_i (FA) \geq \sigma_{d+i}(A) \sigma_{n-d-i+1}(F) = \sigma_{d+i}(A)$ %for $1\leq i\leq n-d$ ($\sigma_i(G)$ denotes the $i$-th largest eigenvalue of $G$). Since $\sigma_{\max}(A) \leq 1$ it follows that $\norm{FA}_{F}^2 %\geq \norm{A}_F^2 - d$.

\paragraph{Lower bound on the ``expectation".} Our first key lemma, is to prove a lower bound on the conditional expectation of each summand of the quantity $\norm{F\vec{f}}_2^2 = \sum_i (F\vec{f})_i^2$ which is captured in Corollary \ref{cor:energy} and is a consequence of the lemma below.

\begin{lemma}[Parity Lemma]\label{lem:oneindex} Fix a sequence of indices $z_1,...,z_{m-1}$ and an index $i$. It holds that
		\[\mathbb{E}_{\theta_0,\beta_0}[ (F\vec{f})_i^2| \vec{y}_{-z_1,...,-z_{m-1} }] \geq \frac{e^{-(B+M \cdot \Theta)(m-1)}}{2^{m-1}} \left(\sum_{j}F_{ij} w_{j,z_1,...,z_{m-1}}\right)^2.\] In case $j = z_t$ for some $t<m$ then $w_{j,z_1,...,z_{m-1}} = 0$.
	\end{lemma}
\begin{proof}
\begin{align*}
&\mathbb{E}_{\theta_0,\beta_0}[ (F\vec{f})_i^2| \vec{y}_{-z_1,...,-z_{m-1}}]  = \mathbb{E}_{\theta_0,\beta_0}\left[\left(\sum_j F_{ij}f_j(\vec{y}_{-j})\right)^2 \big| \vec{y}_{-z_1,...,-z_{m-1}}\right] \\&= \mathbb{E}_{\theta_0,\beta_0}\left[\left(\sum_j F_{ij}\sum_{\vec{e}, j \in \vec{e} }w_{\vec{e}}\vec{y}_{\vec{e}\backslash \{j\}}\right)^2 \big| \vec{y}_{-z_1,...,-z_{m-1}}\right]
\\& = \mathbb{E}_{\theta_0,\beta_0}\left[\left(\sum_j F_{ij}\sum_{\vec{e}: j, z_1 \in \vec{e} }w_{\vec{e}}\vec{y}_{\vec{e\backslash \{j\}}}+ \sum_j F_{ij}\sum_{\vec{e}: j \in \vec{e}, z_1 \notin \vec{e} }w_{\vec{e}}\vec{y}_{\vec{e}\backslash \{j\}}\right)^2 \big| \vec{y}_{-z_1,...,-z_{m-1}}\right]
\\& = \mathbb{E}_{\theta_0,\beta_0}\left[\left(y_{z_1}\sum_{j\neq z_1} \sum_{\vec{e}: j, z_1 \in \vec{e} }F_{ij} w_{\vec{e}}\vec{y}_{\vec{e}\backslash \{j,z_1\}}+ \sum_j \sum_{\vec{e}: j \in \vec{e}, z_1 \notin \vec{e} }F_{ij}w_{\vec{e}}\vec{y}_{\vec{e}\backslash \{j\}} + F_{iz_1}f_{z_1}\right)^2 \big| \vec{y}_{-z_1,...,-z_{m-1}}\right].
\end{align*}
It is clear that the square above is at least $(\sum_j \sum_{\vec{e}: j, z_1 \in \vec{e} }F_{ij} w_{\vec{e}}\vec{y}_{\vec{e}\backslash \{j,z_1\}})^2$ depending on $y_{z_1} = \pm 1$. Thus using the fact that $|f_{z_1}| \leq 1$ we conclude that
\begin{align}\label{eq:step1}
\mathbb{E}_{\theta_0,\beta_0}[ (F\vec{f})_i^2| \vec{y}_{-\vec{e} }] &\geq
\frac{e^{-(B+ M \cdot \Theta)}}{2} \mathbb{E}_{\theta_0,\beta_0}\left[\left(\sum_j \sum_{\vec{e}: j, z_1 \in \vec{e} }F_{ij} w_{\vec{e}}\vec{y}_{\vec{e}\backslash \{j,z_1\}}\right)^2 \big| \vec{y}_{-z_1,...,-z_{m-1}}\right]
\\&= \frac{e^{-(B+ M \cdot \Theta)}}{2} \mathbb{E}_{\theta_0,\beta_0}\left[\left(\sum_j \sum_{\vec{e}: j, z_1 \in \vec{e} }F_{ij} w_{\vec{e}}\vec{y}_{\vec{e}\backslash \{j,z_1\}}\right)^2 \big| \vec{y}_{-z_1,...,-z_{m-1}}\right].
\end{align}
Now observe that (and since $\frac{\partial f_{z_1}}{\partial y_{z_1}}=0$) $\sum_{j\neq z_1} \sum_{\vec{e}: j, z_1 \in \vec{e} }F_{ij} w_{\vec{e}}\vec{y}_{\vec{e}\backslash \{j,z_1\}} = \sum_j F_{ij} \frac{\partial f_j}{\partial y_{z_1}}$ hence we conclude that
\begin{equation}\label{eq:finalstep1}
\mathbb{E}_{\theta_0,\beta_0}[ (F\vec{f})_i^2\big| \vec{y}_{-z_1,...,-z_{m-1}}] \geq \frac{e^{-(B+ M \cdot \Theta)}}{2} \mathbb{E}_{\theta_0,\beta_0}[ (F\tilde{\vec{f}})_i^2\big| \vec{y}_{-z_1,...,-z_{m-1}}],
\end{equation}
where $\tilde{\vec{f}} = (\frac{\partial f_1}{\partial y_{z_1}},...,\frac{\partial f_n}{\partial y_{z_1}})$. By an induction argument we may conclude that
\begin{equation}\label{eq:finalstep1}
\mathbb{E}_{\theta_0,\beta_0}[ (F\vec{f})_i^2\big| \vec{y}_{-z_1,...,-z_{m-1}}] \geq \left(\frac{e^{-(B+ M \cdot \Theta)}}{2}\right)^{m-1} \mathbb{E}_{\theta_0,\beta_0}[ (F\hat{\vec{f}})_i^2\big| \vec{y}_{-z_1,...,-z_{m-1}}],
\end{equation}
where $\hat{f} = (\frac{\partial^{m-1} f_1}{\partial y_{z_1} ... \partial y_{z_{m-1}}},...,\frac{\partial^{m-1} f_n}{\partial y_{z_1} ... \partial y_{z_{m-1}}}) = (\frac{\partial^m f}{\partial y_1 \partial y_{z_1} ... \partial y_{z_{m-1}}},...,\frac{\partial^m f}{\partial y_n \partial y_{z_1} ... \partial y_{z_{m-1}}}) = (w_{1,z_1,...,z_{m-1}},...,w_{n,z_1,...,z_{m-1}} )$.
\end{proof}
\begin{corollary}[Tower property]\label{cor:energy} For each vertex $i$ and distinct vertices $v, z_1,...,z_{m-2}$ (note $i$ is not necessarily different from $v$) it holds
\[\mathbb{E}_{\theta_0,\beta_0}[ (F\vec{f})_i^2| \vec{y}_{-v }] \geq \frac{e^{-(B+M \cdot \Theta)(m-1)}}{2^{m-1}} \left(\sum_{j}F_{ij} w_{\{v,z_1,...,z_{m-2}\} \cup \{j\} }\right)^2.\]
\end{corollary}
\begin{proof} It follows by applying Lemma \ref{lem:oneindex} and the tower property.
\end{proof}
In what follows, we define an ``adjacency" matrix $A$ which enables us to reduce the general degree $m$ polynomial case to the case where $m=2$.
\paragraph{Reduction to ``simple graphs".} To prove strong concavity of the Hessian of the log-pseudolikelihood, we need to show that $\norm{F\vec{f}}_2^2$ is at least $c_1 n$ with high probability. To do this, we reduce the general problem to the case $m=2$ by defining the appropriate matrix below and then use the machinery of \cite{DDP19} to show that $\norm{F\vec{f}}_2^2$ is concentrated around its conditional expectation (see Lemma \ref{lem:concentrationconditional}).

Let $A$ be the following $n \times n$ matrix: For each column $i$ of $A$, let \[(z_{1}^*,...,z_{m-2}^*) = \textrm{argmax}_{z_1,...,z_{m-2}} \norm{(w_{(z_1,...,z_{m-2},i,1)},...,w_{(z_1,...,z_{m-2},i,n)})}_2.\]
The $j$-th entry of column $i$ of $A$ is given by $w_{(z_1^*,...,z_{m-2}^*,i,j)}$. Intuitively, matrix $A$ induces a subgraph of the original hypergraph $G$. Nevertheless, matrix $A$ contains ``enough edges" to infer $\theta_0,\beta_0$.
\begin{lemma}[$A$ has big Frobenius norm]\label{lem:Abig} There exists a constant $C$ such that
\begin{equation}
\norm{A}_F^2 \geq Cn.
\end{equation}
\end{lemma}
\begin{proof}
Define the matrix $B$ of size $|E(G)_m| \times n$ where $E(G)_m$ is the set of edges of cardinality $m$, $B_{\vec{e},i} = w_{e} \times \vec{1}_{i \in \vec{e}}$ and $\vec{e} \in E(G)_m$. It holds that $\norm{B}_F^2$ is $\Omega(n)$. Consider the maximum entry in absolute value per column of $B$, let $b_i$, i.e., $b_i = \norm{B^i}_{\infty}$. Since $\norm{B^i}_{1} \leq 1$ (bounded degree assumption) by Holder's inequality we get that $b_i \geq \norm{B^i}_2^2$. Therefore we conclude that $\sum_i b_i \geq \norm{B}_F^2$, thus it is $\Omega(n)$. From Cauchy-Schwarz we get that $\sum_i b_i^2 \geq \frac{(\sum_i b_i)^2}{n} \geq \frac{\norm{B}_F^4}{n}$ which is $\Omega(n)$.

The proof is complete by observing that $\norm{A^i}_2^2 \geq b_i^2$ for all $i$ and thus $\norm{A}_F^2$ is $\Omega (n).$
\end{proof}
\noindent
Moreover, $A$ satisfies the bounded degree condition and this is captured by the lemma below.

\begin{lemma}[Bounding $\norm{A}_{\infty}, \norm{A}_{1}$]\label{lem:Asmallinfinity} It holds that \[\norm{A}_{1}, \norm{A}_{\infty} \leq m-1.\]
\end{lemma}
\begin{proof}
 Each entry in $A_{ij}$ is some weight of an edge that contains $i,j$ (if there exists one otherwise zero). Hence in every row/column, each edge appears at most $m-1$ times and by the bounded degree assumption the claim follows.
\end{proof}

Note that from Lemma \ref{lem:Abig} and $\ref{lem:Asmallinfinity}$ we get $\norm{FA}_F^2$ is also $\Omega (n)$. This is true, since $\norm{A}_2 \leq \sqrt{\norm{A}_1\norm{A}_{\infty}} \leq m-1$ (Lemma \ref{lem:Asmallinfinity}), thus $\norm{FA}_F^2 \geq \norm{A}_F^2 - d(m-1)^2$.
To proceed, we use a selection index procedure that appeared in \cite{DDP19} (we mention it below for completeness) and which will be useful in the later part of the proof.
	\paragraph{An Index Selection Procedure \cite{DDP19}:} Given a matrix $W$, we define $h: [n]\to [n]$ as follows. Consider the following iterative process. At time $t=0$, we start with the $n\times n$ matrix, $W^1 = W$. At time step $t$ we choose from $W^t$ the row with maximum $\ell_2$ norm (let $i_t$ the index of that row, ties broken arbitrarily) and also let $j_t = \textrm{argmax}_j |W^t_{i_tj}| $ (again ties broken arbitrarily). We set $h(i_t) = j_t$ and $W^{t+1}$ is $W^t$ by setting zeros the entries of $i_t^{th}$ row and column $j_t^{th}$. We run the process above for $n$ steps to define the \textbf{bijection} $h$. The following lemma is taken from \cite{DDP19}.
	\begin{lemma}[\cite{DDP19}]
		\label{lem:sum-of-maxima-lb}
		Assume that $\norm{FA}_{\infty} \leq c_{\infty}'$ and\footnote{Recall $F = I - X(X^{\top}X)^{-1}X^{\top}$.} $\norm{FA}_F^2 \geq c_F n$ for some positive constant $c_{\infty}, c_F$ and $\norm{A}_2, \norm{A}_{\infty}, \norm{A}_1$ are also bounded. We run the process described above on $FA$ and get the function $h$. There exists a constant $C$ (depends on $c_F, c_{\infty}$) such that
		\[  \sum_i |(FA)_{ih(i)}|^2  \ge Cn.\]
	\end{lemma}
Combining Corollary \ref{cor:energy} (summing over all $i$) with Lemma \ref{lem:sum-of-maxima-lb}, there exists a constant $C$ (independent of $n,d$) such that the following inequality is true (always)
\begin{equation}\label{eq:lowerboundlinear}
\sum_{i} \mathbb{E}_{\theta_0,\beta_0}[ (F\vec{f})_i^2| \vec{y}_{-h(i) }]\geq \frac{e^{-(B+M \cdot \Theta)(m-1)}}{2^{m-1}} \times \sum_i (FA)_{ih(i)}^2 \geq C \times \frac{e^{-(B+M \cdot \Theta)(m-1)}}{2^{m-1}} \times n  .
\end{equation}
Equation (\ref{eq:lowerboundlinear}) gives us the linear in $n$ lower bound that we want for the sum of conditional expectations of the terms $(F\vec{f})_i^2$. Finally we need to show that the term $\sum_{i} (F\vec{f})_i^2$ is not far from $\sum_{i} \mathbb{E}_{\theta_0,\beta_0}[ (F\vec{f})_i^2| \vec{y}_{-h(i) }]$ with high probability, thus it is also at least linear in $n$ and Lemma \ref{lem:main2} would follow. This is captured in the following lemma.
\begin{lemma}[Bounding the ``conditional" variance]
		\label{lem:concentrationconditional}
		It holds that
		\[\mathbb{E}_{\theta_0,\beta_0}\left[\left(\sum_{i=1}^n (F\vec{f})^2_i - \sum_{i=1}^n \mathbb{E}_{\theta_0,\beta_0}\left[(F\vec{f})_i^2 | \vec{y}_{-h(i)}\right] \right)^2\right] \le (80n+16Bn)(m-1). \]
	\end{lemma}
\paragraph{Putting it all together}\label{sec:alltogether}
\begin{proof}[Proof of Theorem \ref{thm:logistic}]
We can prove now our main result, the approach is similar to \cite{DDP19}. 	From Corollary \ref{cor:smallgradient}	we get that (for some constant $C_1$)
	\begin{align}
	\Pr \left[\norm{\nabla LPL(\theta_0,\beta_0) }^2_{2} \le  \frac{C_1 d}{n\delta}  \right] \ge 1- \delta.
	\end{align}
	for any constant $\delta$.
	Next, we have from Lemma~\ref{lem:main2} and the analysis in the beginning of Section \ref{sec:convexity} that, $\min_{(\theta, \beta) \in \mathbb{B}} \lambda_{\min}\left(-H_{(\theta,\beta)}\right) \ge C_2$ for some constant $C_2$ independent of $n,d$.
	Plugging into Lemma~\ref{lem:logistic-consistency}, we get that
	\begin{align}
	\norm{(\theta_{\mathcal{D}} - \theta_0,\beta_{\mathcal{D}}-\beta_0)}_2  = \mathcal{D} \norm{(\hat{\theta}_{MPL} - \theta_0,\hat{\beta}_{MPL}-\beta_0)}_2 \le \frac{\norm{\nabla LPL(\theta_0,\beta_0) }_2}{\min_{(\theta, \beta) \in \mathbb{B}} \lambda_{\min}\left(-H_{(\theta,\beta)}\right)} \label{eq:lc5}
	\end{align}
	Now we have from the above that $\norm{(\theta_{\mathcal{D}} - \theta_0,\beta_{\mathcal{D}}-\beta_0)}_2 \to 0$ as $n \to \infty$ and also holds that $\norm{(\theta_{\mathcal{D}} - \theta_0,\beta_{\mathcal{D}}-\beta_0)}_2 \to 0$ which implies that $\mathcal{D} = 1$ for sufficiently large $n$. Therefore
	\begin{align}
	 \eqref{eq:lc5} \implies \norm{(\hat{\theta}_{MPL} - \theta_0,\hat{\beta}_{MPL}-\beta_0)}_2 &\le \frac{\norm{\nabla LPL(\theta_0,\beta_0) }_2}{\min_{(\theta, \beta) \in \mathbb{B}} \lambda_{\min}\left(-H_{(\theta,\beta)}\right)} \\
	 &\le O\left(\sqrt{\frac{d}{n}}\right)
	\end{align}
	with probability $\ge 1-\delta$. The analysis of Projected Gradient Descent can be found in the appendix.
\end{proof}

\section{Conclusion}
\label{sec:conclusion} 
In this paper, we focused on the problem of parameter estimation from one sample of a high dimensional discrete distribution that can be viewed as an instantiation Logistic Regression from dependent observations or Inference on Ising models, with high-order peer effects. There are many open questions, we state a few:
\begin{itemize}
\item In the consistency rate, there is an exponential dependence on the degree $m$ of the polynomial function $f$ ($m$ now is considered a constant number). Can this be improved?
\item Analyze more complicated settings where function $f$ is Lipschitz.
\end{itemize} 
\section{Acknowledgements} Constantinos Daskalakis and Nishanth Dikkala were supported by NSF Awards IIS-1741137, CCF-1617730 and CCF-1901292, by a Simons Investigator Award, by the DOE PhILMs project (No. DE-AC05-76RL01830), by the DARPA award HR00111990021, by a Google Faculty award, by the MIT Frank Quick Faculty Research and Innovation Fellowship, and an MIT-IBM Watson AI Lab research grant. Ioannis Panageas was supported by SRG ISTD 2018 136, NRF-NRFFAI1-2019-0003 and NRF2019NRF-ANR2019.

\bibliographystyle{plain}
%\bibliography{ising}
\bibliography{ising,biblio,biblio2}

\appendix
\section{Missing Proofs}
\begin{proof}[Proof of Lemma \ref{lem:conc1}]
	We use the powerful technique of exchangeable pairs as introduced by Chatterjee and employed by Chatterjee and Dembo. First it holds by assumption that it trivially follows that $|f_i(\vec{y}_{-i})| \leq 1$ for all $i$ and $\vec{y}_{-i} \in \{-1,+1\}^{n-1}$.
	Set
	\begin{equation}
	Q(\vec{y}) := \sum_{i} (y_i -\tanh (\beta_0 f_i(\vec{y}_{-i})+\theta_0^{\top}\vec{x}_i))f_i(\vec{y}_{-i}),
	\end{equation}
	hence we get
	\begin{align}
	\frac{\partial Q(\vec{y})}{\partial y_j} = \sum_{i} \left(\vec{1}_{i=j} - \frac{\beta_0 \frac{\partial f_i (\vec{y}_{-i})}{\partial y_j}}{\cosh^2(\beta_0 f_i(\vec{y}_{-i})+\theta_0^{\top}\vec{x}_i)}\right)f_i(\vec{y}_{-i}) + \left(y_i -\tanh (\beta_0 f_i(\vec{y}_{-i})+\theta_0^{\top}\vec{x}_i)\right)\frac{\partial f_i(\vec{y}_{-i})}{\partial y_j}.
	\end{align}
	We will bound the absolute value of each summand. First observe that $\left|\frac{\partial f_i(\vec{y}_{-i})}{\partial y_j}\right|\leq \sum_{\vec{e}: i,j \in \vec{e}} |w_{\vec{e}}|$, hence we can bound the second term as follows
	\begin{equation}\label{eq:sum2}
	\left|(y_i -\tanh (\beta_0 f_i(\vec{y}_{-i})+\theta_0^{\top}\vec{x}_i))\frac{\partial f_i(\vec{y}_{-i})}{\partial y_j}\right| \leq 2 \sum_{\vec{e}: i,j \in e} |w_{\vec{e}}|.
	\end{equation}
	Using the fact that $\frac{1}{\cosh^2(x)} \leq 1$ it also follows that
	\begin{equation}\label{eq:sum1}
\begin{array}{cc}
	\left|\sum_{i} \left(\vec{1}_{i=j} - \frac{\beta_0 \frac{\partial f_i (\vec{y}_{-i})}{\partial y_j}}{\cosh^2(\beta_0 f_i(\vec{y}_{-i})+\theta_0^{\top}\vec{x}_i)}\right)f_i(\vec{y}_{-i})\right| \leq |f_j(\vec{y}_{-j})| + \sum_{i\neq j}  \sum_{\vec{e}: i,j \in \vec{e}} |\beta_0||w_{\vec{e}}f_i(\vec{y}_{-i})| \\\leq |f_j(\vec{y}_{-j})| + \sum_{i\neq j}  \sum_{\vec{e}: i,j \in \vec{e}} |\beta_0||w_{\vec{e}}|.
\end{array}
\end{equation}
	Using (\ref{eq:sum2}) and  (\ref{eq:sum1}) it follows that $\left|\frac{\partial Q(\vec{y})}{\partial y_j}\right| \leq \sum_{i\neq j} \sum_{\vec{e}: i,j \in e} |w_{\vec{e}}|(2+|\beta_0|) + |f_j(\vec{y}_{-j})|$.
	Finally let $\vec{y^j} = (\vec{y}_{-j}, -1)$ and note that
	\begin{align}\label{eq:boundflip1}
	|Q(\vec{y}) - Q(\vec{y^j})| &\leq 2 \cdot \left(\sum_{i\neq j} \sum_{\vec{e}: i,j \in e} |w_{\vec{e}}|(2+|\beta_0|) + \max_{\vec{y}_{-j}}|f_j(\vec{y}_{-j})|\right) \\&\leq 2 \cdot (1+(2+B)(m-1)\sum_{\vec{e}: j \in \vec{e}}|w_{\vec{e}}|)\\&\leq (4+2B)(m-1)+2 \leq (6+2B)(m-1).
	\end{align}
	We have all the ingredients to complete the proof. We first observe that
	\begin{equation}
	\sum_i \mathbb{E}_{\theta_0,\beta_0}[(y_i - \tanh (\beta_0 f_i(\vec{y}_{-i})+\theta_0^{\top}\vec{x}_i))Q(\vec{y^i})f_i(\vec{y}_{-i})]=0,
	\end{equation}
	since
	\begin{equation}
	\begin{array}{cc}
	\mathbb{E}_{\theta_0,\beta_0}[(y_i - \tanh (\beta_0 f_i(\vec{y}_{-i})+\theta_0^{\top}\vec{x}_i))Q(\vec{y^i})f_i(\vec{y}_{-i})]= \\= \mathbb{E}_{\theta_0,\beta_0}[\mathbb{E}[(y_i - \tanh (\beta_0 f_i(\vec{y}_{-i})+\theta_0^{\top}\vec{x}_i))Q(\vec{y^i})f_i(\vec{y}_{-i})| \vec{y}_{-i}]] =0.
	\end{array}
	\end{equation}
	Therefore it follows
	\begin{align*}
	\mathbb{E}_{\theta_0,\beta_0}[Q^2(\vec{y})] &= \mathbb{E}_{\theta_0,\beta_0}\left[Q(\vec{y}) \cdot \left(\sum_{i} (y_i -\tanh (\beta_0 f_i(\vec{y}_{-i})+\theta_0^{\top}\vec{x}_i))f_i(\vec{y}_{-i})\right)\right] \\& = \mathbb{E}_{\theta_0,\beta_0}\left[\sum_{i} \left(Q(\vec{y}) (y_i -\tanh (\beta_0 f_i(\vec{y}_{-i})+\theta_0^{\top}\vec{x}_i))f_i(\vec{y}_{-i})\right)\right] \\& = \sum_{i} \mathbb{E}_{\theta_0,\beta_0}\left[(Q(\vec{y}) - Q(\vec{y^i})) \cdot  (y_i -\tanh (\beta_0 f_i(\vec{y}_{-i})+\theta_0^{\top}\vec{x}_i))f_i(\vec{y}_{-i})\right] \\& \leq \sum_i 2 \cdot (6+2B)(m-1) = (12+4B)(m-1)n.
	\end{align*}
\end{proof}

\begin{proof}[Proof of Lemma \ref{lem:conc2}]
	We fix a coordinate $k$ and set
	\begin{equation}
	Q(\vec{y}) := \sum_{i} (y_i -\tanh (\beta_0 f_i(\vec{y}_{-i})+\theta_0^{\top}\vec{x}_i))x_{i,k},
	\end{equation}
	
	hence we get $\frac{\partial Q(\vec{y})}{\partial y_j} = \sum_{i} \left(\vec{1}_{i=j} - \frac{\beta_0 \frac{\partial f_i (\vec{y}_{-i})}{\partial y_j}}{\cosh^2(\beta_0 f_i(\vec{y}_{-i})+\theta_0^{\top}\vec{x}_i)}\right)x_{i,k}.$
	We will bound the term as follows
	\begin{equation}\label{eq:sum3}
	\left|\frac{\partial Q(\vec{y})}{\partial y_j} \right| \leq |x_{j,k}| + \sum_{i\neq j} |\beta_0||x_{i,k}| \sum_{\vec{e}: i,j \in \vec{e}} |w_{\vec{e}}|.
	\end{equation}
	
	Finally let $\vec{y^j} = (\vec{y}_{-j}, -1)$ and note that
	\begin{equation}\label{eq:boundflip2}
	|Q(\vec{y}) - Q(\vec{y^j})| \leq 2 \cdot \left(|x_{j,k}| + \sum_{i\neq j} |\beta_0||x_{i,k}| \sum_{\vec{e}: i,j \in \vec{e}} |w_{\vec{e}}|\right).
	\end{equation}
	We have all the ingredients to complete the proof. We first observe that
	\begin{equation}
	\sum_i \mathbb{E}_{\theta_0,\beta_0}[(y_i - \tanh (\beta_0 f_i(\vec{y}_{-i})+\theta_0^{\top}\vec{x}_i))Q(\vec{y^i})x_{i,k}]=0,
	\end{equation}
	since
	\begin{equation}
	\begin{array}{cc}
	\mathbb{E}_{\theta_0,\beta_0}[(y_i - \tanh (\beta_0 f_i(\vec{y}_{-i})+\theta_0^{\top}\vec{x}_i))Q(\vec{y^i})x_{i,k}]= \\= \mathbb{E}_{\theta_0,\beta_0}[\mathbb{E}[(y_i - \tanh (\beta_0 f_i(\vec{y}_{-i})+\theta_0^{\top}\vec{x}_i))Q(\vec{y^i})x_{i,k}| \vec{y}_{-i}]] =0.
	\end{array}
	\end{equation}
	Therefore it follows
	\begin{align*}
	\mathbb{E}_{\theta_0,\beta_0}[Q^2(\vec{y})] &= \mathbb{E}_{\theta_0,\beta_0}\left[Q(\vec{y}) \cdot \left(\sum_{i} (y_i -\tanh (\beta_0 f_i(\vec{y}_{-i})+\theta_0^{\top}\vec{x}_i))x_{i,k}\right)\right] \\& = \mathbb{E}_{\theta_0,\beta_0}\left[\sum_{i} \left(Q(\vec{y}) (y_i -\tanh (\beta_0 f_i(\vec{y}_{-i})+\theta_0^{\top}\vec{x}_i))x_{i,k}\right)\right] \\& = \sum_{i} \mathbb{E}_{\theta_0,\beta_0}\left[(Q(\vec{y}) - Q(\vec{y^i})) \cdot  (y_i -\tanh (\beta_0 f_i(\vec{y}_{-i})+\theta_0^{\top}\vec{x}_i))x_{i,k}\right] \\& \leq \sum_i 4 \cdot (x_{i,k}^2 + |x_{i,k}|\sum_{j \neq i} |\beta_0| |x_{j,k}| \sum_{\vec{e}: i,j \in \vec{e}}|w_{\vec{e}}| ) \\& \leq 4\sum_{i} |x_{i,k}|^2 + B|x_{i,k}| \max_{j} |x_{j,k}| \sum_{j\neq i}\sum_{\vec{e}: i,j \in \vec{e}}|w_{\vec{e}}| \\&\leq 4M^2 n + \sum_{i} 4BM^2 (m-1) \sum_{\vec{e}: i \in \vec{e}} |w_{\vec{e}}| = 4nM^2(1 + B(m-1))\\&\leq 4n(m-1)M^2(1+B),
	\end{align*}
	and the claim follows by summing over all the coordinates.
\end{proof}
	\begin{proof}[Proof of Lemma \ref{lem:concentrationconditional}]
		For each $i$, we expand the term $\mathbb{E}_{\theta_0,\beta_0}\left[(F\vec{f})_i^2 | \vec{y}_{-h(i)}\right]$ and we get
		$\mathbb{E}_{\theta_0,\beta_0}\left[(F\vec{f})_i^2 | \vec{y}_{-h(i)}\right] = \mathbb{E}_{\theta_0,\beta_0}\left[\left(y_{h(i)}\sum_{j\neq h(i)} \sum_{\vec{e}: j, h(i) \in \vec{e} }F_{ij} w_{\vec{e}}\vec{y}_{\vec{e}\backslash \{j,h(i)\}}+ \sum_j \sum_{\vec{e}: j \in \vec{e}, h(i) \notin \vec{e} }F_{ij}w_{\vec{e}}\vec{y}_{\vec{e}\backslash \{j\}} + F_{ih(i)}f_{h(i)}\right)^2| \vec{y}_{-h(i)}\right]$. We set $z_{it}(\vec{y}) = 2\left(\sum_{j\neq t} \sum_{\vec{e}: j, t \in \vec{e} }F_{ij} w_{\vec{e}}\vec{y}_{\vec{e}\backslash \{j,t\}}\right) \left (\sum_j \sum_{\vec{e}: j \in \vec{e}, t \notin \vec{e} }F_{ij}w_{\vec{e}}\vec{y}_{\vec{e}\backslash \{j\}} + F_{it}f_{t}\right)$ (does not depend on $y_t$) and we get that the expectation we need to bound is equal to
		\[\mathbb{E}_{\theta_0,\beta_0}\left[\left(\sum_i z_{ih(i)}(\vec{y})y_{h(i)} - z_{ih(i)}(\vec{y})\tanh\left(\beta_0 f_{h(i)}(\vec{y})+ \theta_0^{\top}\vec{x}_{h(i)}\right) \right)^2\right].\]
		First it holds that $\frac{\partial z_{it}}{y_j} = 2 \left(\sum_{j' \neq t} \sum_{\vec{e}: j',j,t \in \vec{e}} F_{ij'}w_{\vec{e}}y_{\vec{e} \backslash\{t,j,j'\}}\right) \left (\sum_{j'} \sum_{\vec{e}: j' \in \vec{e}, t \notin \vec{e} }F_{ij'}w_{\vec{e}}\vec{y}_{\vec{e}\backslash \{j'\}} + F_{it}f_{t}\right) +$ \\$+ 2\left(\sum_{j'\neq t} \sum_{\vec{e}: j', t \in \vec{e} }F_{ij'} w_{\vec{e}}\vec{y}_{\vec{e}\backslash \{j',t\}}\right) \left (\sum_{j'} \sum_{\vec{e}: j',j \in \vec{e}, t \notin \vec{e} }F_{ij'}w_{\vec{e}}\vec{y}_{\vec{e}\backslash \{j',j\}} + F_{it} \frac{\partial f_{t}}{y_t}\right)$ and $\frac{\partial z_{it}}{y_t} = 0$. Also by the bounded degree condition it holds that $|z_{it}| \leq 4 $ as long as $\norm{F}_{\infty}$ is bounded by one. The rest of the proof follows as in Lemma 3.7 in \cite{DDP19} (using exchangeable pairs).
	\end{proof}
\section{Projected Gradient Descent}
\label{sec:gd-analysis}
The following is a well-known fact for Projected Gradient Descent (Theorem 3.10 from \cite{bubeck2015convex}).
\begin{theorem}\label{thm:projected} Let $f$ be $\alpha$-strongly convex and $\lambda$-smooth on compact set $\mathcal{X}$. Then
projected gradient descent with stepsize $\eta = \frac{1}{\lambda}$ satisfies for $t\geq 0$
\begin{equation}
\norm{\vec{x}_{t+1} - \vec{x}^*}_2^2 \leq e^{-\frac{\alpha t}{\lambda}}\norm{\vec{x}_{1} - \vec{x}^*}_2^2.
\end{equation}
Therefore, setting $R = \norm{\vec{x}_{1} - \vec{x}^*}_2$ and by choosing $t = \frac{2\lambda \ln \frac{R}{\epsilon}}{\alpha}$ it is guaranteed that $\norm{\vec{x}_{t+1} - \vec{x}^*}_2 \leq \epsilon$.
\end{theorem}
We consider the function $LPL(\theta,\beta)$ (log-pseudolikelihood as defined in Section \ref{sec:pseudolike}) and we would like to approximate $(\hat{\theta},\hat{\beta})$ within $\frac{1}{\sqrt{n}}$ in $\ell_2$ distance. The stepsize in Theorem \ref{thm:projected} should be $\eta=\frac{1}{\Theta^2 +1}$ by Remark \ref{rem:smooth}.

\begin{algorithm}[H]
	\label{alg:pgd-logistic}
	\KwData{Vector sample $\vec{y}$, ``Magnetizations" $f_i(\vec{y}_{-i}) = y_{i} \sum_{\vec{e}: i \in \vec{e}} w_{\vec{e}}y_{\vec{e}}$, Feature vectors $\vec{x}_i$}
	\KwResult{Maximum Pseudolikelihood Estimate}
	$\beta^0 =0, \theta^0 = \vec{0}, \textrm{normgrad} = +\infty$, $\eta = \frac{1}{\Theta^2 +1}$\;
	$t = 0$\;
	\While{$\textrm{normgrad} > \frac{1}{\sqrt{n}}$}{
		$\textrm{grad}_{\theta}=0$\;
		$\textrm{grad}_{\beta} = -\frac{1}{n}\sum_{i=1}^n \left[ y_i f_{i}(\vec{y}) -  f_i(\vec{y})\tanh (\beta^t f_i(\vec{y})+\theta^{t \;\top}\vec{x}_i)\right]$\;
		\For{$k=1;k\leq d;k++$}
		{
			$\textrm{grad}_{\theta_k} = -\frac{1}{n}\sum_{i=1}^n \left[ y_ix_{i,k} - x_{i,k}\tanh (\beta^t f_i(\vec{y})+\theta^{t \;\top}\vec{x}_i)\right]$\;
			$\textrm{grad}_{\theta} = \textrm{grad}_{\theta} + \textrm{grad}^2_{\theta_k}$\;
		}
		$\textrm{normgrad} = \sqrt{\textrm{grad}^2_{\beta}+\textrm{grad}_{\theta}}$\;
		
		$\beta^{t+1} = \beta^t -  \eta\textrm{grad}_{\beta}$    \% update $\beta^t$\;
		\For{$k=1;k\leq d;k++$}
		{
			$\theta^{t+1}_k = \theta^t_k - \eta\textrm{grad}_{\theta_k}$     \% update $\theta^t_k$\;
		}
		\% $\ell_2$ projection\\
		\If{$\beta^{t+1} <-B$}{
			$\beta^{t+1} = -B$\;
		}
		\If{$\beta^{t+1} > B $}{
			$\beta^{t+1} = B $\;
		}
        $\textrm{norm}_{\theta}=0$\;
		\For{$k=1;k \leq d;k++$}{
			$\textrm{norm}_{\theta}=\textrm{norm}_{\theta}+(\theta^{t+1}_{k})^2$\;
			}
		\If{$\sqrt{\textrm{norm}_{\theta}}> \Theta$}{
				$\theta^{t+1} = \theta^{t+1}\frac{\Theta}{\sqrt{\textrm{norm}_{\theta}}}$\;
			}
		$t = t+1$\;	
		}
	\Return{$(\theta^t,\beta^t)$}
	\caption{Projected Gradient Descent}
\end{algorithm}

\end{document}